\documentclass{article}
\usepackage{nips10submit_e,times}
\usepackage[reqno]{amsmath}
\usepackage{amsxtra, amsfonts,amscd,  amssymb, graphicx,subfigure}
\usepackage[numbers,square,sort]{natbib}
\usepackage{algorithm,algorithmic}
\newcommand{\be}{\begin{equation}}
\newcommand{\ee}{\end{equation}}
\newcommand{\ba}{\begin{array}}
\newcommand{\ea}{\end{array}}
\newcommand{\bea}{\begin{eqnarray}}
\newcommand{\eea}{\end{eqnarray}}
\newcommand{\beaa}{\begin{eqnarray*}}
\newcommand{\eeaa}{\end{eqnarray*}}

\newcommand{\br}{\mathbb{R}}

\newcommand{\sgn}{\mathrm{sgn}}

\newcommand{\etal}{{et al. }}
\newcommand{\Diag}{\mbox{Diag}}

\usepackage{amsthm}
\newtheorem{theorem}{Theorem}[section]
\newtheorem{lemma}[theorem]{Lemma}

\nipsfinalcopy 

\begin{document}

\title{Sparse Inverse Covariance Selection  via \\  Alternating Linearization Methods}
\author{Katya Scheinberg \\ Department of ISE \\ Lehigh University \\ \texttt{katyas@lehigh.edu} \\
\And Shiqian Ma, \quad Donald Goldfarb \\ Department of IEOR \\ Columbia University \\ \texttt{\{sm2756,goldfarb\}@columbia.edu}}



\maketitle

\begin{abstract} Gaussian graphical models are of great interest in statistical learning. Because the conditional independencies between different nodes correspond to zero entries in the inverse covariance matrix of the Gaussian distribution, one can learn the structure of the graph by estimating a sparse inverse covariance matrix from sample data, by solving a convex maximum likelihood problem with an $\ell_1$-regularization term.
In this paper, we propose a first-order
method based on an alternating linearization technique that exploits the problem's special structure; in particular, the subproblems solved in each iteration have closed-form solutions. Moreover, our algorithm obtains an $\epsilon$-optimal solution in $O(1/\epsilon)$ iterations. Numerical experiments on both synthetic  and real data from gene association networks show that a {practical version of
 this} algorithm outperforms other competitive algorithms.
\end{abstract}

\section{Introduction} \label{sec:intro}
In multivariate data analysis, graphical models such as Gaussian Markov Random Fields provide a way to discover meaningful interactions among variables. Let $Y=\{y^{(1)},\ldots,y^{(n)}\}$ be an $n$-dimensional random vector following an $n$-variate Gaussian distribution $\mathcal{N}(\mu,\Sigma)$, and let
$G=(V,E)$ be a Markov network representing the conditional
independence structure of $\mathcal{N}(\mu,\Sigma)$. Specifically, the set of vertices $V=\{1,\ldots,n\}$ corresponds to the set of variables in $Y$, and the edge set $E$
contains an edge $(i,j)$ if and only if $y^{(i)}$ is conditionally
dependent on $y^{(j)}$ given all remaining variables; i.e., the lack of an
edge between $i$ and $j$ denotes the conditional independence of $y^{(i)}$ and $y^{(j)}$,
which corresponds to a zero entry in the inverse covariance matrix $\Sigma^{-1}$ (\cite{Lauritzen-1996}).
Thus learning the structure of this graphical model is equivalent to the problem of learning the zero-pattern of $\Sigma^{-1}$.
To estimate this sparse inverse covariance matrix, one can solve the following sparse inverse covariance selection (SICS) problem:
$ \max_{X\in S^n_{++}} \log\det(X) - \langle\hat\Sigma, X\rangle - \rho \|X\|_0,$
where $S^n_{++}$ denotes the set of $n\times n$ positive definite matrices, $\|X\|_0$ is the number of nonzeros in $X$,
$\hat\Sigma = \frac{1}{p}\sum_{i=1}^p (Y_i-\hat\beta)(Y_i-\hat\beta)^\top$ is the sample covariance matrix, $\hat\beta=\frac{1}{p}\sum_{i=1}^p Y_i$ is the sample mean and $Y_i$ is the $i$-th random sample of $Y$. This problem is NP-hard in general due to the combinatorial nature of the cardinality term $\rho\|X\|_0$ (\cite{Natarajan-95}). To get a numerically tractable problem, {one can replace the cardinality term $\|X\|_0$ by $\|X\|_1:=\sum_{i,j}|X_{ij}|$, the envelope of $\|X\|_0$ over the set $\{X\in\br^{n\times n}: \|X\|_\infty \leq 1\}$ }(see
\cite{Hiriart-Urruty-Lemarechal-1993}). This results in the convex optimization problem (see e.g., \cite{Yuan-Lin-2007,Friedman-Hastie-Tibshirani-2007,Wainwright-Ravikumar-Lafferty-2007,Banerjee-ElGhaoui-Aspremont-2008}):
\bea\label{prob:sics-L1}\min_{X\in S^n_{++}} \quad -\log\det(X) +
\langle\hat\Sigma, X\rangle + \rho \|X\|_1.\eea
Note that \eqref{prob:sics-L1} can be rewritten as $ \min_{X\in S^n_{++}}\max_{\|U\|_\infty\leq\rho}-\log\det X+\langle\hat\Sigma+U,X\rangle,$ where $\|U\|_\infty$ is the largest absolute value of the entries of $U$. By exchanging the order of max and min, we obtain the dual problem
$\max_{\|U\|_\infty\leq\rho}\min_{X\in S^n_{++}}-\log\det X+\langle\hat\Sigma+U,X\rangle,$ which is equivalent to
\bea\label{prob:sics-L1-dual}\max_{W\in S^n_{++}} \{ \log\det W +n :  \|W-\hat\Sigma\|_\infty\leq\rho\}.\eea
Both the primal and dual problems have strictly convex objectives; hence, their optimal solutions are unique. Given a dual solution $W$, $X=W^{-1}$ is primal feasible resulting in the duality gap
\bea\label{def:gap-Winv} gap:=\langle\hat\Sigma,W^{-1}\rangle+\rho\|W^{-1}\|_1-n.\eea

The primal and the dual SICS problems \eqref{prob:sics-L1} and \eqref{prob:sics-L1-dual}
    are semidefinite programming problems and can be solved via interior point methods (IPMs) in polynomial time. However, the per-iteration computational cost and memory
 requirements of an IPM are prohibitively high for the SICS problem. Although an approximate IPM has recently been proposed for the SICS problem \cite{Li-Toh-2010},  most of the methods developed for it are first-order methods.
  Banerjee \etal \cite{Banerjee-ElGhaoui-Aspremont-2008} proposed a block coordinate descent (BCD) method to solve the dual problem \eqref{prob:sics-L1-dual}. Their method updates one row and one column of $W$ in each iteration by solving a convex quadratic programming problem by an IPM. The $glasso$ method of Friedman \etal \cite{Friedman-Hastie-Tibshirani-2007} is  based on the same BCD approach as in \cite{Banerjee-ElGhaoui-Aspremont-2008}, but it solves each subproblem as a
LASSO problem by yet another coordinate descent (CD) method \cite{Tibshirani-LASSO-1996}. Sun \etal \cite{Sun-KDD09-SICS}
proposed solving the primal problem \eqref{prob:sics-L1} by using a BCD method. They formulate
the subproblem as a min-max problem and solve it using a prox method proposed by Nemirovski
\cite{Nemirovski-Prox-siopt-2005}. The SINCO method proposed by Scheinberg and Rish \cite{Scheinberg-Rish-2009} is a greedy CD method applied to the primal problem. All of these BCD and CD approaches lack
iteration complexity bounds. They also have been shown to be inferior in practice to gradient based
approaches. A projected gradient method for solving the dual problem \eqref{prob:sics-L1-dual} that is considered to be state-of-the-art for SICS was proposed by
 Duchi \etal \cite{Duchi-UAI-2008}. However, there are no iteration complexity results for it either.
Variants of Nesterov's method \cite{Nesterov-2005,NesterovConvexBook2004}
 have been applied to solve the SICS problem. 
 d'Aspremont \etal \cite{Aspremont-Banerjee-ElGhaoui-2008} applied
Nesterov's optimal first-order method to solve the primal problem \eqref{prob:sics-L1} after
smoothing the nonsmooth $\ell_1$ term, obtaining an iteration complexity bound of $O(1/\epsilon)$ for an $\epsilon$-optimal solution, but the implementation in \cite{Aspremont-Banerjee-ElGhaoui-2008} was very slow and did not produce good results. Lu \cite{Lu-covsel-siopt-2009} solved the dual problem \eqref{prob:sics-L1-dual}, which
is a smooth problem, by Nesterov's algorithm, and improved the iteration complexity to $O(1/\sqrt{\epsilon})$. However, since the practical performance of this algorithm was not attractive, Lu gave a variant (VSM) of it that exhibited better performance. The iteration complexity of VSM is unknown. Yuan \cite{Yuan-2009} proposed an alternating direction method based on an augmented Lagrangian framework {(see the ADAL method \eqref{alg:ADAL} below)}. This method also lacks complexity results. The proximal point algorithm proposed by Wang \etal in \cite{Wang-Sun-Toh-2009} requires a reformulation of the problem that increases the size of the problem making it impractical for solving large-scale problems. Also, there is no iteration complexity bound for this algorithm. The IPM in \cite{Li-Toh-2010} also requires such a reformulation.



{\bf Our contribution.} In this paper, we propose an alternating linearization method (ALM) for solving the primal SICS problem.
An advantage of solving the primal problem is that the $\ell_1$ penalty term in the objective function directly promotes sparsity in
the optimal inverse covariance matrix.

 Although developed independently, our method is closely related to
Yuan's method  \cite{Yuan-2009}. Both methods exploit the special form of the primal problem
 \eqref{prob:sics-L1} by alternatingly minimizing one of the terms of the objective  {function}
plus an  approximation to the other  term. The main difference between the two methods is in the construction of these approximations. As we will show, our method has a
theoretically justified interpretation and is based on an algorithmic framework with complexity
bounds, while no complexity bound is available for Yuan's method. Also our method has an intuitive interpretation from a learning perspective.
 Extensive numerical test results on both synthetic data and real problems have shown that our ALM algorithm significantly outperforms other existing algorithms, such as the PSM algorithm proposed by Duchi \etal \cite{Duchi-UAI-2008} and the VSM algorithm proposed by Lu \cite{Lu-covsel-siopt-2009}. Note that it is shown in \cite{Duchi-UAI-2008} and \cite{Lu-covsel-siopt-2009} that PSM and VSM outperform the BCD method in \cite{Banerjee-ElGhaoui-Aspremont-2008} and $glasso$ in \cite{Friedman-Hastie-Tibshirani-2007}.

{\bf Organization of the paper.} In Section 2 we briefly review alternating linearization methods for minimizing the sum of two convex functions and establish convergence and iteration complexity results.
We show how to use ALM to solve SICS problems and give intuition from a learning perspective in Section 3. Finally, we present some numerical results on both synthetic and real data in Section 4 and compare ALM with PSM algorithm \cite{Duchi-UAI-2008} and VSM algorithm \cite{Lu-covsel-siopt-2009}.

\section{Alternating Linearization Methods} \label{sec:alm}
We consider here the alternating linearization method (ALM) for solving the following problem: \bea\label{prob:min-sum-two} \min \quad F(x)
\equiv f(x) + g(x), \eea where $f$ and $g$ are both convex functions.  An effective way to solve \eqref{prob:min-sum-two} is to ``split'' $f$ and $g$ by introducing a new variable, i.e., to rewrite \eqref{prob:min-sum-two} as
\bea\label{prob:min-sum-two-X-Y}\min_{x,y} \{ f(x) + g(y) :  x-y=0 \},\eea and apply an alternating direction augmented Lagrangian method to it. Given a penalty parameter
 $1/\mu$, at the $k$-th iteration, the augmented Lagrangian method minimizes the augmented Lagrangian function $$\mathcal{L}(x,y;\lambda):=f(x)+g(y)-\langle\lambda,x-y\rangle+\frac{1}{2\mu}\|x-y\|_2^2,$$ with respect to $x$ and $y$, i.e., it solves the subproblem \bea\label{prob:min-Aug-Lag-XY-nonsmooth}(x^{k},y^{k}):=\arg\min_{x,y} \mathcal{L}(x,y;\lambda^k),\eea and updates the Lagrange multiplier $\lambda$ via:
\bea\label{update-Lambda}\lambda^{k+1}:=\lambda^k-(x^{k}-y^{k})/\mu.\eea Since minimizing $\mathcal{L}(x,y;\lambda)$ with respect to $x$ and $y$ jointly is usually difficult, while doing so with respect to $x$ and $y$ alternatingly can often be done efficiently, the following alternating direction version of the augmented Lagrangian method (ADAL) is often advocated (see, e.g., \cite{Fortin-Glowinski-1983,Glowinski-LeTallec-89}):
\bea\label{alg:ADAL}\left\{\ba{ll}x^{k+1} & :=\arg\min_x \mathcal{L}(x,y^k;\lambda^k) \\ y^{k+1} & := \arg\min_y \mathcal{L}(x^{k+1},y;\lambda^k) \\ \lambda^{k+1} & := \lambda^k - (x^{k+1}-y^{k+1})/\mu. \ea\right.\eea

If we also update
$\lambda$ after we solve the subproblem with respect to $x$, we get the
following symmetric version of the ADAL method.
\bea\label{alg:sym-ADAL}\left\{\ba{ll} x^{k+1} & :=
\arg\min_x \mathcal{L}(x,y^k;\lambda_y^k) \\
                                       \lambda_x^{k+1} &:=\lambda_y^k-(x^{k+1}-y^k)/\mu \\
                                       y^{k+1} & := \arg\min_y \mathcal{L}(x^{k+1},y;\lambda_x^{k+1}) \\
                                       \lambda_y^{k+1} & := \lambda_x^{k+1} -(x^{k+1}-y^{k+1})/\mu. \ea\right.\eea
Algorithm \eqref{alg:sym-ADAL} has certain theoretical advantages when $f$ and $g$ are smooth. In this case, from the first-order optimality conditions for the two subproblems in
\eqref{alg:sym-ADAL}, we have that:
\bea\label{ADAL-opt-cond}\lambda_x^{k+1}=\nabla f(x^{k+1}) \quad \mbox{and} \quad \lambda_y^{k+1} = -\nabla g(y^{k+1}).\eea Substituting these relations into \eqref{alg:sym-ADAL}, we obtain the following equivalent  algorithm for solving \eqref{prob:min-sum-two}, which we refer to as the alternating linearization minimization (ALM) algorithm.
\begin{algorithm}[tbh!]\caption{Alternating linearization method (ALM) for smooth problem}\label{alg:ALM-general}
\begin{algorithmic}
\STATE {\bfseries Input:} $x^0=y^0$
\FOR{$k=0,1,\cdots$}
\STATE 1. Solve $ x^{k+1}:= \arg\min_{x} Q_g(x,y^k) \equiv f(x) + g(y^k) + \left\langle \nabla g(y^k), x-y^k \right\rangle + \frac{1}{2\mu}\|x-y^k\|_2^2$;
\STATE 2. Solve $ y^{k+1}:= \arg\min_{y} Q_f(x^{k+1},y) \equiv f(x^{k+1}) + \left\langle \nabla f(x^{k+1}), y-x^{k+1} \right\rangle + \frac{1}{2\mu}\|y-x^{k+1}\|_2^2 + g(y) $;
\ENDFOR
\end{algorithmic}
\end{algorithm}

Algorithm \ref{alg:ALM-general} can be viewed in the following way: at each iteration  we
construct a quadratic approximation of the function $g(x)$ at the current iterate $y^k$
and minimize the sum of  this approximation and $f(x)$.
The approximation is based on linearizing $g(x)$ (hence the name ALM) and adding a ``prox'' term $\frac{1}{2\mu}\|x-y^k\|_2^2$.
When $\mu$ is small enough ($\mu\leq 1/L(g)$, where $L(g)$ is the Lipschitz constant for $\nabla g$)  this quadratic function,
$g(y^k) + \left\langle \nabla g(y^k), x-y^k \right\rangle +
\frac{1}{2\mu}\|x-y^k\|_2^2$
is an upper approximation to $g(x)$, which means that the reduction in the value of $F(x)$
achieved by minimizing $Q_g(x,y^k)$ in Step 1 is not smaller than the reduction achieved in the value
of  $Q_g(x,y^k)$ itself.  Similarly, in Step 2 we build an upper
approximation to $f(x)$ at $x^{k+1}$, $f(x^{k+1}) + \left\langle \nabla f(x^{k+1}), y-x^{k+1} \right\rangle + \frac{1}{2\mu}\|y-x^{k+1}\|_2^2,$
and minimize the sum $Q_f(x^{k+1},y)$ of it and $g(y)$.

Let us now assume that  $f(x)$ is in the class $C^{1,1}$ with Lipschitz constant
$L(f)$, while $g(x)$ is simply convex. Then from the first-order optimality conditions for the second minimization in \eqref{alg:sym-ADAL}, we have $-\lambda_y^{k+1} \in \partial g(y^{k+1})$, the subdifferential of $g(y)$ at $y=y^{k+1}$. Hence, replacing $\nabla g(y^k)$ in the definition of $Q_g(x, y^k)$ by $-\lambda_y^{k+1}$ in \eqref{alg:sym-ADAL}, we obtain the following modified version of \eqref{alg:sym-ADAL}.

\begin{algorithm}[tbh!]\caption{Alternating linearization method with skipping step}\label{alg:ALM-nonsmooth}
\begin{algorithmic}
\STATE {\bfseries Input:} $x^0=y^0$
\FOR{$k=0,1,\cdots$}
\STATE 1. Solve $ x^{k+1}:= \arg\min_{x} Q(x,y^k) \equiv f(x) + g(y^k) - \left\langle \lambda^k, x-y^k \right\rangle + \frac{1}{2\mu}\|x-y^k\|_2^2$;
\STATE 2. If $  F(x^{k+1})> Q(x^{k+1}, y^k)$ then  $x^{k+1}  := y^k$.
\STATE 3. Solve $ y^{k+1}:= \arg\min_{y} Q_f(x^{k+1},y) $;
\STATE 4. $\lambda^{k+1}= \nabla f(x^{k+1})-(x^{k+1}-y^{k+1})/\mu$.
\ENDFOR
\end{algorithmic}
\end{algorithm}

Algorithm \ref{alg:ALM-nonsmooth} is identical to the symmetric ADAL algorithm \eqref{alg:sym-ADAL} as long as
$F(x^{k+1})\leq Q(x^{k+1}, y^k)$ at each iteration (and to Algorithm \ref{alg:ALM-general} if $g(x)$ is in $C^{1,1}$ and $\mu\leq 1/\max\{L(f),L(g)\}$). If this condition fails, then the algorithm simply sets $x^{k+1}\leftarrow y^k$.
Algorithm \ref{alg:ALM-nonsmooth} has the following convergence property and iteration complexity bound.
For a proof see the Appendix.

\begin{theorem}\label{the:ALM-nonsmooth}
Assume $\nabla f$  is Lipschitz continuous with constant $L(f)$. For $\beta/L(f) \leq\mu \leq 1/L(f)$ where $0<\beta\leq 1$,
Algorithm \ref{alg:ALM-nonsmooth} satisfies
\begin{align}\label{the:ALM-nonsmooth-conclude} F(y^k)-F(x^*)\leq\frac{\|x^0-x^*\|^2}{2\mu (k+k_n)}, \forall k, \end{align} where $x^*$ is an optimal solution of \eqref{prob:min-sum-two} and $k_n$ is the number of
iterations until the $k-th$ for which $F(x^{k+1})\leq Q(x^{k+1}, y^k)$. Thus Algorithm \ref{alg:ALM-nonsmooth} produces a sequence which converges to the optimal solution in function value, and the number of iterations needed is $O(1/\epsilon)$ for an $\epsilon$-optimal solution.
\end{theorem}

If $g(x)$ is also a smooth function in the class $C^{1,1}$ with Lipschitz constant $L(g)\leq 1/\mu$, then Theorem \ref{the:ALM-nonsmooth} also applies to Algorithm \ref{alg:ALM-general} since in this case $k_n=k$ (i.e., no ``skipping'' occurs).
Note that the iteration complexity bound in Theorem \ref{the:ALM-nonsmooth} can be improved. Nesterov \cite{Nesterov-1983,NesterovConvexBook2004} proved that  one can obtain an optimal iteration complexity bound
of  $O(1/\sqrt{\epsilon})$, using only first-order information.
 His acceleration technique is based on using a linear combination of previous
iterates to obtain a point where the approximation is built.
This technique has been exploited and extended by Tseng \cite{Tseng-2008}, Beck and Teboulle \cite{Beck-Teboulle-2009}, Goldfarb \etal \cite{Goldfarb-Ma-Scheinberg-2010} and many others. A similar technique can be adopted to derive a fast version of Algorithm \ref{alg:ALM-nonsmooth} that has an improved complexity bound of $O(1/\sqrt{\epsilon})$, while keeping the computational effort in each iteration almost unchanged.
However, we do not present this method here, since when applied to the SICS problem, it did not work as well as Algorithm \ref{alg:ALM-nonsmooth}.

\section{ALM for SICS}
The SICS problem \bea\label{prob:min-f-g-SICS} \min_{X\in S^n_{++}} \quad F(X) \equiv f(X) + g(X), \eea where $f(X)=-\log\det(X)+\langle\hat\Sigma,X\rangle$ and $g(X)=\rho\|X\|_1$, is of the same form as \eqref{prob:min-sum-two}. However, in this case
neither $f(X)$ nor $g(X)$ have Lipschitz continuous gradients. Moreover,
$f(X)$ is only defined for positive definite matrices while $g(X)$  is defined
everywhere. These properties of the  objective function  make the
SICS problem especially challenging for optimization methods. Nevertheless,
 we can still
 apply \eqref{alg:sym-ADAL} to solve the problem directly.
Moreover, we can apply Algorithm \ref{alg:ALM-nonsmooth} and obtain the complexity bound in  Theorem \ref{the:ALM-nonsmooth}  as follows.

The $\log\det(X)$ term in $f(X)$ implicitly requires that $X\in S^n_{++}$ and
the gradient of $f(X)$, which is given by $-X^{-1}+\hat\Sigma$, is not Lipschitz continuous in $S^n_{++}$. Fortunately, as proved in Proposition 3.1 in \cite{Lu-covsel-siopt-2009}, the optimal solution of \eqref{prob:min-f-g-SICS} $X^*\succeq \alpha I$, where $\alpha = \frac{1}{\|\hat\Sigma\|+n\rho}.$
 Therefore, if we define $\mathcal{C}:=\{X\in S^n: X\succeq\frac{\alpha}{2} I\}$, the SICS problem  \eqref{prob:min-f-g-SICS} can be formulated as:
\bea\label{prob:min-f-g-constraint}\min_{X,Y} \{ f(X) + g(Y) : X - Y = 0, X \in \mathcal{C}, Y \in \mathcal{C}\}.\eea 

 We can include constraints  $X \in \mathcal{C}$ in Step 1 and  $Y \in \mathcal{C}$
in Step 3 of  Algorithm \ref{alg:ALM-nonsmooth}.
 Theorem \ref{the:ALM-nonsmooth} can  then be applied as discussed in \citep{Goldfarb-Ma-Scheinberg-2010}. However, a difficulty now
arises when performing the minimization in $Y$. {Without the constraint  $Y \in \mathcal{C}$,
only a matrix shrinkage operation is needed, but with this additional constraint the problem becomes harder to solve.}
Minimization in $X$  with or without
the constraint  $X \in \mathcal{C}$ is accomplished by performing an SVD. Hence the constraint can be easily imposed.

Instead of imposing constraint  $Y \in \mathcal{C}$  we can obtain feasible solutions by a line search on $\mu$.
 We know that the constraint $ X\succeq\frac{\alpha}{2} I$ is not tight at the solution. Hence if we start
the algorithm with $X\succeq\alpha I $ and restrict the step size $\mu$ to be sufficiently small
then the iterates of the method will remain  in $\mathcal{C}$.

Note however, that the bound on the Lipschitz constant of the gradient of $f(X)$ is $1/\alpha^2$ and hence can be very large. It is not practical to restrict $\mu$ in the algorithm to be smaller than $\alpha^2$, since $\mu$ determines the step size at each iteration. Hence, for a practical approach we can only claim that the theoretical convergence rate bound holds in only a small neighborhood of the optimal solution. We now present a practical version of our algorithm applied to the SICS problem.
\begin{algorithm}[tbh!]\caption{Alternating linearization method (ALM) for SICS}\label{alg:ALM-SICS-constraint}
\begin{algorithmic}
\STATE {\bfseries Input:} $X^0=Y^0$, $\mu_0$.
\FOR{$k=0,1,\cdots$}
\STATE 0. Pick $\mu_{k+1}\leq \mu_{k}$.
\STATE 1. Solve $ X^{k+1}:= \arg\min_{X\in \mathcal{C}} f(X) + g (Y^k) - \langle\Lambda^k, X-Y^k\rangle + \frac{1}{2\mu_{k+1}}\|X-Y^k\|_F^2$;
\STATE 2.  If $  g(X^{k+1})> g (Y^k) - \langle\Lambda^k,X^{k+1}-Y^k\rangle + \frac{1}{2\mu_{k+1}}\|X^{k+1}-Y^k\|_F^2$, then  $X^{k+1}  := Y^k$.
\STATE 3. Solve $ Y^{k+1}:= \arg\min_{Y} f(X^{k+1}) + \langle\nabla f(X^{k+1}),Y-X^{k+1}\rangle + \frac{1}{2\mu_{k+1}}\|Y-X^{k+1}\|_F^2 + g(Y)$;
\STATE 4. $\Lambda^{k+1}=\nabla f(X^{k+1})- (X^{k+1}-Y^{k+1})/\mu_{k+1}$.
\ENDFOR
\end{algorithmic}
\end{algorithm}

We now show how to solve the two optimization problems in Algorithm \ref{alg:ALM-SICS-constraint}. The first-order optimality conditions for Step 1 in Algorithm
\ref{alg:ALM-SICS-constraint}, ignoring the constraint $X\in\mathcal{C}$ are:
\bea\label{optcond-ALM-X}\nabla f(X)-\Lambda^k+(X-Y^k)/\mu_{k+1}=0.\eea
Consider $V\Diag(d)V^\top$ - the spectral decomposition of
$Y^k+\mu_{k+1}(\Lambda^k-\hat{\Sigma})$ and let
 \bea\label{X-sub-gamma}\gamma_i=\left(d_i+\sqrt{d_i^2+4\mu_{k+1}}\right)/2, i=1,\ldots,n.\eea
 Since $\nabla f(X)=-X^{-1}+\hat\Sigma$, it is easy to verify that $X^{k+1}:=V\Diag(\gamma)V^\top$ satisfies
\eqref{optcond-ALM-X}.  When the constraint $X\in\mathcal{C}$ is imposed, the
optimal solution  changes  to  $X^{k+1}:=V\Diag(\gamma)V^\top$  with $\gamma_i=\max\left\{\alpha/2,\left(d_i+\sqrt{d_i^2+4\mu_{k+1}}\right)/2\right\}, i=1,\ldots,n.$
We observe that solving \eqref{optcond-ALM-X} requires approximately the same effort
($O(n^3)$) as is required to  compute $\nabla f(X^{k+1})$. Moreover, from the solution to
\eqref{optcond-ALM-X}, $\nabla f(X^{k+1})$ is obtained with only a negligible amount
of additional effort, since
$(X^{k+1})^{-1}:=V\Diag(\gamma)^{-1}V^\top$.

The first-order optimality conditions for Step 2 in Algorithm \ref{alg:ALM-SICS-constraint}
are:
\bea\label{optcond-ALM-Y} 0 \in \nabla f(X^{k+1})+(Y-X^{k+1})/\mu_{k+1}+\partial g(Y).\eea Since
$g(Y)=\rho \|Y\|_1$, it is well known that the solution to \eqref{optcond-ALM-Y} is given by
 \beaa
Y^{k+1}={\rm shrink}( X^{k+1}-\mu_{k+1}(\hat\Sigma -(X^{k+1})^{-1}), \mu_{k+1}\rho ),
\eeaa
where the ``shrinkage operator'' ${\rm shrink}(Z, \rho)$ updates each element $Z_{ij}$ of the matrix $Z$ by the formula ${\rm shrink}(Z, \rho)_{ij}=\sgn(Z_{ij})\cdot\max\{|Z_{ij}|-\rho,0\}.$


The $O(n^3)$ complexity of Step 1, which requires a spectral decomposition, dominates the $O(n^2)$ complexity of Step 2 which requires a simple shrinkage.
There is no closed-form solution for the subproblem
corresponding to $Y$ when the constraint $Y\in\mathcal{C}$ is imposed. Hence, we neither impose
this constraint explicitly nor do so by a line search on $\mu_k$, since in practice this degrades the performance of the
algorithm substantially. Thus, the resulting iterates $Y^k$ may not be positive definite,
while the iterates $X^k$ remain so. Eventually due to the convergence of  $Y^k$ and $X^k$, the $Y^k$ iterates become positive definite and the constraint $Y\in\mathcal{C}$ is satisfied.


Let us now remark on the learning based intuition behind Algorithm \ref{alg:ALM-SICS-constraint}.
We recall that $-\Lambda^k \in \partial g(Y^k)$. The two steps of the algorithm can be written as
\bea\label{eq:stepX}
 X^{k+1}:= \arg\min_{X\in \mathcal{C}} \{f(X)+\frac{1}{2\mu_{k+1}}\|X-(Y^k+\mu_{k+1} \Lambda^k)\|_F^2\}
\eea
and
\bea\label{eq:stepY}
Y^{k+1}:= \arg\min_{Y} \{ g(Y)+\frac{1}{2\mu_{k+1}}\|Y-(X^{k+1}-\mu_{k+1} (\hat\Sigma -(X^{k+1})^{-1}))\|_F^2\}.
\eea
The SICS problem is trying to optimize two conflicting objectives:
on the one hand it tries to find a  covariance matrix $X^{-1}$ that best fits the observed data,
 i.e., is as close to $\hat\Sigma$ as possible, and on the other hand it tries to obtain a sparse matrix $X$. The proposed algorithm address these two objectives in an alternating manner.
Given an initial ``guess'' of the sparse matrix $Y^k$ we update this guess by a
subgradient descent step of length $\mu_{k+1}$: $Y^k+\mu_{k+1}\Lambda^k$. Recall that $-\Lambda^k \in \partial g(Y^k)$.
Then problem (\ref{eq:stepX})  seeks a solution $X$ that optimizes the first objective (best fit of the data)  while adding a regularization term which imposes a Gaussian prior on $X$ whose mean is
the current {\em guess} for the sparse matrix: $Y^k+\mu_{k+1} \Lambda^k$.
The solution to  (\ref{eq:stepX}) gives us  a {\em guess} for the inverse covariance $X^{k+1}$.
We again update it by taking a gradient descent step: $X^{k+1}-\mu_{k+1} (\hat\Sigma -(X^{k+1})^{-1})$. Then problem
\eqref{eq:stepY}  seeks a sparse solution $Y$ while also imposing a Gaussian prior on  $Y$
whose mean is the {\em guess} for the inverse covariance matrix
$X^{k+1}-\mu_{k+1} (\hat\Sigma -(X^{k+1})^{-1})$. Hence the sequence of $X^k$'s is a
sequence of
positive definite  inverse covariance matrices that converge to a sparse matrix,
while the sequence of $Y^k$'s  is a sequence of sparse matrices that converges to a positive definite
inverse covariance matrix.

An important question is how to pick $\mu_{k+1}$. Theory tells us that if we pick a small enough value, then we can obtain the complexity bounds. However, in
practice this value is too small. We discuss the simple strategy that we use in the next section.

\section{Numerical Experiments} \label{sec:numerical}
In this section, we present numerical results on both synthetic and real data to demonstrate the efficiency of our SICS ALM algorithm. Our codes for ALM were written in MATLAB. All numerical experiments were run in MATLAB 7.3.0 on a Dell Precision 670 workstation with an Intel Xeon(TM) 3.4GHZ CPU and 6GB of RAM.

Since $-\Lambda^k\in\partial g(Y^k)$, $\|\Lambda^k\|_\infty\leq \rho$; hence $\hat\Sigma-\Lambda^k$ is a feasible solution to the dual problem \eqref{prob:sics-L1-dual} as long as it is positive
definite. Thus the duality gap at the $k$-th iteration is given by:
\bea\label{dual-gap} Dgap := -\log\det(X^k)+\langle\hat\Sigma,X^k\rangle+\rho\|X^k\|_1-\log\det(\hat\Sigma-\Lambda^k)-n. \eea We define the relative duality gap as:
 $Rel.gap := Dgap/(1+|pobj|+|dobj|),$ where $pobj$ and $dobj$ are respectively the objective function values of the primal problem \eqref{prob:min-f-g-SICS} at point $X^k$, and the dual problem \eqref{prob:sics-L1-dual} at $\hat\Sigma-\Lambda^k$. Defining $d_k(\phi(x))\equiv\max\{1,\phi(x^k),\phi(x^{k-1})\}$, we measure the relative changes of objective function value $F(X)$ and the iterates $X$ and $Y$ as follows: \beaa Frel:= \frac{|F(X^{k})-F(X^{k-1})|}{d_k(|F(X)|)},\ Xrel := \frac{\|X^k-X^{k-1}\|_F}{d_k(\|X\|_F)}, \ Yrel := \frac{\|Y^k-Y^{k-1}\|_F}{d(\|Y\|_F)}.\eeaa

We terminate ALM when either \bea\label{stop-crit-dgap} (i) \quad Dgap \leq \epsilon_{gap} \quad \mbox{ or } \quad (ii) \quad \max\{Frel,Xrel,Yrel\} \leq \epsilon_{rel}. \eea
Note that in \eqref{dual-gap}, computing $\log\det(X^k)$ is easy since the spectral decomposition of $X^k$ is already available (see \eqref{optcond-ALM-X} and \eqref{X-sub-gamma}), but computing $\log\det(\hat\Sigma-\Lambda^k)$ requires another expensive spectral decomposition. Thus, in practice, we only check \eqref{stop-crit-dgap}(i) every $N_{gap}$ iterations. We check \eqref{stop-crit-dgap}(ii) at every iteration since this is inexpensive.

A continuation strategy for updating $\mu$ is also crucial to ALM. In our experiments, we adopted the following update rule. After every $N_{\mu}$ iterations, we set $\mu:=\max\{\mu\cdot\eta_\mu, \bar{\mu}\}$; i.e., we simply reduce $\mu$ by a constant factor $\eta_\mu$ every $N_\mu$ iterations until a desired lower bound on $\mu$ is achieved.

We compare ALM (i.e., Algorithm \ref{alg:ALM-SICS-constraint} with the above stopping criteria and $\mu$ updates), with the projected subgradient method (PSM) proposed by Duchi \etal in \cite{Duchi-UAI-2008} and implemented by Mark Schmidt \footnote{The MATLAB can be downloaded from http://www.cs.ubc.ca/$\sim$schmidtm/Software/PQN.html} and the smoothing method (VSM) \footnote{The MATLAB code can be downloaded from http://www.math.sfu.ca/$\sim$zhaosong} proposed by Lu in \cite{Lu-covsel-siopt-2009}, which are considered to be  the state-of-the-art algorithms for solving SICS problems. The per-iteration complexity of all three algorithms is roughly the same; hence a comparison of the number of iterations is meaningful. The parameters used in PSM and VSM are set at their default values.
We used the following parameter values in ALM: $\epsilon_{gap}=10^{-3},\epsilon_{rel}=10^{-8},N_{gap}=20,N_{\mu}=20,
\bar{\mu}=\max\{\mu_0\eta_{\mu}^8,10^{-6}\},\eta_\mu=1/3,$ where $\mu_0$ is the initial $\mu$ which is set according to $\rho$; specifically, in our experiments, $\mu_0=100/\rho,$ if $\rho<0.5$,
$\mu_0=\rho$ if $0.5\leq \rho\leq 10$, and $\mu_0=\rho/100$ if $\rho>10$.

\subsection{Experiments on synthetic data}
We randomly created test problems using a procedure proposed by Scheinberg and Rish in \cite{Scheinberg-Rish-2009}. Similar procedures were used by Wang \etal in \cite{Wang-Sun-Toh-2009} and Li and Toh in \cite{Li-Toh-2010}.
For a given dimension $n$, we first created a sparse matrix $U\in \br^{n\times n}$ with nonzero entries equal to -1 or 1 with equal probability. Then we computed $S:=(U*U^\top)^{-1}$ as the true covariance matrix. Hence, $S^{-1}$ was sparse. We then drew $p=5n$ iid vectors, $Y_1,\ldots,Y_p$, from the Gaussian distribution $\mathcal{N}(\mathbf{0}, S)$ by using the $mvnrnd$ function in MATLAB, and computed a sample covariance matrix $\hat\Sigma:=\frac{1}{p}\sum_{i=1}^pY_iY_i^\top.$ We compared ALM with PSM \cite{Duchi-UAI-2008} and VSM \cite{Lu-covsel-siopt-2009} on these randomly created data with different $\rho$. The PSM code was terminated using its default stopping criteria, which included \eqref{stop-crit-dgap}(i) with $\epsilon_{gap}=10^{-3}$. VSM was also terminated when $Dgap\leq 10^{-3}$. Since PSM and VSM solve the dual problem \eqref{prob:sics-L1-dual}, the duality gap which is given by \eqref{def:gap-Winv} is available without any additional spectral decompositions. The results are shown in Table \ref{tab:ALM-Duchi-K}. All CPU times reported are in seconds.

\begin{table}[ht]
\begin{center}\caption{Comparison of ALM, PSM and VSM on synthetic data}\label{tab:ALM-Duchi-K}{\scriptsize
\begin{tabular}{|c||c|c|c|c||c|c|c|c||c|c|c|c|}\hline
  & \multicolumn{4}{|c||}{ALM} & \multicolumn{4}{|c||}{PSM} & \multicolumn{4}{|c|}{VSM}\\\hline
n & iter & Dgap & Rel.gap & CPU & iter & Dgap & Rel.gap & CPU & iter & Dgap & Rel.gap & CPU \\\hline
\multicolumn{13}{|c|}{$\rho=0.1$} \\\hline
200      & 300  & 8.70e-4 & 1.51e-6 & 13   & 1682    & 9.99e-4 & 1.74e-6 & 38    &  857 & 9.97e-4 & 1.73e-6 & 37   \\\hline
500      & 220  & 5.55e-4 & 4.10e-7 & 84   & 861     & 9.98e-4 & 7.38e-7 & 205   &  946 & 9.98e-4 & 7.38e-7 & 377 \\\hline
1000     & 180  & 9.92e-4 & 3.91e-7 & 433  & 292     & 9.91e-4 & 3.91e-7 & 446   &  741 & 9.97e-4 & 3.94e-7 & 1928 \\\hline
1500     & 199  & 1.73e-3 & 4.86e-7 & 1405 & 419     & 9.76e-4 & 2.74e-7 & 1975  &  802 & 9.98e-4 & 2.80e-7 & 6340 \\\hline
2000     & 200  & 6.13e-5 & 1.35e-8 & 3110 & 349     & 1.12e-3 & 2.46e-7 & 3759  &  915 & 1.00e-3 & 2.20e-7 & 16085 \\\hline
\multicolumn{13}{|c|}{$\rho=0.5$} \\\hline
200     & 140     &  9.80e-4  & 1.15e-6  &  6      & 6106    & 1.00e-3  & 1.18e-6   & 137  & 1000 & 9.99e-4 & 1.18e-6 & 43 \\\hline
500     & 100     &  1.69e-4  & 7.59e-8  &  39     & 903     & 9.90e-4  & 4.46e-7   & 212  & 1067 & 9.99e-4 & 4.50e-7 & 425 \\\hline
1000    & 100     &  9.28e-4  & 2.12e-7  &  247    & 489     & 9.80e-4  & 2.24e-7   & 749  & 1039 & 9.95e-4 & 2.27e-7 & 2709 \\\hline
1500    & 140     &  2.17e-4  & 3.39e-8  &  1014   & 746     & 9.96e-4  & 1.55e-7   & 3514 & 1191 & 9.96e-4 & 1.55e-7 & 9405  \\\hline
2000    & 160     &  4.70e-4  & 5.60e-8  &  2529   & 613     & 9.96e-4  & 1.18e-7   & 6519 & 1640 & 9.99e-4 & 1.19e-7 & 28779  \\\hline
\multicolumn{13}{|c|}{$\rho=1.0$} \\\hline
200      & 180    & 4.63e-4   & 4.63e-7  & 8     & 7536    & 1.00e-3 & 1.00e-6   & 171     & 1296 & 9.96e-4 & 9.96e-7 & 57      \\\hline
500      & 140    & 4.14e-4   & 1.56e-7  & 55    & 2099    & 9.96e-4 & 3.76e-7   & 495    & 1015 & 9.97e-4 & 3.76e-7 & 406     \\\hline
1000     & 160    & 3.19e-4   & 6.07e-8  & 394   & 774     & 9.83e-4 & 1.87e-7   & 1172   & 1310 & 9.97e-4 & 1.90e-7 & 3426    \\\hline
1500     & 180    & 8.28e-4   & 1.07e-7  & 1304  & 1088    & 9.88e-4 & 1.27e-7   & 5100   & 1484 & 9.96e-4 & 1.28e-7 & 11749  \\\hline
2000     & 240    & 9.58e-4   & 9.37e-8  & 3794  & 1158    & 9.35e-4 & 9.15e-8   & 12310  & 2132 & 9.99e-4 & 9.77e-8 & 37406   \\\hline
\end{tabular}}
\end{center}
\end{table}

From Table \ref{tab:ALM-Duchi-K} we see that on these randomly created SICS problems, ALM outperforms PSM and VSM in both accuracy and CPU time with the performance gap increasing as $\rho$ increases. For example, for $\rho=1.0$ and $n=2000$, ALM achieves $Dgap=9.58e-4$ in about 1 hour and 15 minutes, while PSM and VSM need about 3 hours and 25 minutes and 10 hours and 23 minutes, respectively, to achieve similar accuracy.

\subsection{Experiments on real data}\label{sec:real-data}
We tested ALM on real data from gene expression networks using the five data sets from \cite{Li-Toh-2010} provided to us by Kim-Chuan Toh:
{\it (1) Lymph node status; (2) Estrogen receptor; (3) Arabidopsis thaliana; (4) Leukemia; (5) Hereditary breast cancer.} See \cite{Li-Toh-2010} and references therein for the descriptions of these data sets.
Table \ref{tab:real-data-Toh} presents our test results. As suggested in \cite{Li-Toh-2010}, we set $\rho =0.5$. From Table \ref{tab:real-data-Toh} we see that ALM is much faster and provided more accurate solutions than PSM and VSM.
\begin{table}[ht!]
\begin{center}\caption{Comparison of ALM, PSM and VSM on real data}\label{tab:real-data-Toh}{\scriptsize
\begin{tabular}{|c|c||c|c|c|c||c|c|c|c|c|c|c|c|}\hline
&  & \multicolumn{4}{|c||}{ALM} & \multicolumn{4}{|c|}{PSM}& \multicolumn{4}{|c|}{VSM}\\\hline
prob.  & n     & iter & Dgap & Rel.gap & CPU & iter & Dgap & Rel.gap & CPU & iter & Dgap & Rel.gap & CPU \\\hline
(1)    & 587   & 60   & 9.41e-6  & 5.78e-9 & 35   & 178   & 9.22e-4 & 5.67e-7 & 64     & 467  & 9.78e-4 & 6.01e-7 & 273  \\\hline
(2)    & 692   & 80   & 6.13e-5  & 3.32e-8 & 73   & 969   & 9.94e-4 & 5.38e-7 & 531    & 953  & 9.52e-4 & 5.16e-7 & 884  \\\hline
(3)    & 834   & 100  & 7.26e-5  & 3.27e-8 & 150  & 723   & 1.00e-3 & 4.50e-7 & 662    & 1097 & 7.31e-4 & 3.30e-7 & 1668 \\\hline
(4)    & 1255  & 120  & 6.69e-4  & 1.97e-7 & 549  & 1405  & 9.89e-4 & 2.91e-7 & 4041   & 1740 & 9.36e-4 & 2.76e-7 & 8568 \\\hline
(5)    & 1869  & 160  & 5.59e-4  & 1.18e-7 & 2158 & 1639  & 9.96e-4 & 2.10e-7 & 14505  & 3587 & 9.93e-4 & 2.09e-7 & 52978 \\\hline
\end{tabular}}
\end{center}
\end{table}

\subsection{Solution Sparsity}
In this section, we compare the sparsity patterns of the solutions produced by ALM, PSM and VSM. For ALM, the sparsity of the solution is given by the sparsity of $Y$. Since PSM and VSM solve the dual problem, the primal solution $X$, obtained by inverting the dual solution $W$, is never sparse due to floating point errors. Thus it is not fair to measure the sparsity of $X$ or a truncated version of $X$. Instead, we measure the sparsity of solutions produced by PSM and VSM by appealing to complementary slackness. Specifically, the $(i,j)$-th element of the inverse covariance matrix is deemed to be nonzero if and only if $|W_{ij}-\hat\Sigma_{ij}|=\rho$. We give results for a random problem ($n=500$) and the first real data set in Table \ref{tab:sparsity-comparison}. For each value of $\rho$, the first three rows show the number of nonzeros in the solution and the last three rows show the number of entries that are nonzero in the solution produced by one of the methods but are zero in the solution produced by the other method. The sparsity of the ground truth inverse covariance matrix of the synthetic data is 6.76\%.
\begin{table}[ht!]
\begin{center}\caption{Comparison of sparsity of solutions produced by ALM, PSM and VSM}\label{tab:sparsity-comparison}{\scriptsize
\begin{tabular}{|c|c|c|c|c|c|c|c|c|c|}\hline
$\rho$ & 100 & 50 & 10 & 5 & 1 & 0.5 & 0.1 & 0.05 & 0.01 \\\hline
\multicolumn{10}{|c|}{synthetic problem data} \\\hline
ALM & 700 & 2810 & 11844 & 15324 & 28758 & 37510 & 63000 & 75566 & 106882 \\\hline
PSM & 700 & 2810 & 11844 & 15324 & 28758 & 37510 & 63000 & 75566 & 106870 \\\hline
VSM & 700 & 2810 & 11844 & 15324 & 28758 & 37510 & 63000 & 75568 & 106876 \\\hline
ALM vs PSM & 0 & 0 & 0 & 0 & 0 & 0 & 0 & 2 & 14 \\\hline
PSM vs VSM & 0 & 0 & 0 & 0 & 0 & 0 & 0 & 0 & 8 \\\hline
VSM vs ALM & 0 & 0 & 0 & 0 & 0 & 0 & 0 & 2 & 2 \\\hline
\multicolumn{10}{|c|}{real problem data} \\\hline
ALM & 587 & 587 & 587 & 587 & 587 & 4617 & 37613 & 65959 & 142053 \\\hline
PSM & 587 & 587 & 587 & 587 & 587 & 4617 & 37615 & 65957 & 142051 \\\hline
VSM & 587 & 587 & 587 & 587 & 587 & 4617 & 37613 & 65959 & 142051 \\\hline
ALM vs PSM & 0 & 0 & 0 & 0  & 0   & 0    & 0     & 2     & 2      \\\hline
PSM vs VSM & 0 & 0 & 0 & 0  & 0   & 0    & 2     & 0     & 0      \\\hline
VSM vs ALM & 0 & 0 & 0 & 0  & 0   & 0    & 0     & 0     & 0      \\\hline
\end{tabular}}
\end{center}
\end{table}
From Table \ref{tab:sparsity-comparison} we can see that when $\rho$ is relatively large ($\rho\geq 0.5$), all three algorithms produce solutions with exactly
the same sparsity patterns. Only when $\rho$ is very small, are there slight differences.
We note that the ROC curves depicting the trade-off between the number of true positive elements recovered versus the number of false positive elements as a function of the regularization parameter $\rho$ are also almost identical for the three methods.

\section*{Acknowledgements} We would like to thank Professor Kim-Chuan Toh for providing the data set used in Section \ref{sec:real-data}. The research reported here was supported in
part by NSF Grants DMS 06-06712 and DMS 10-16571, ONR Grant
N00014-08-1-1118 and DOE Grant DE-FG02-08ER25856.

\newpage
{\small
\bibliographystyle{unsrt}
\bibliography{C:/Mywork/Optimization/work/reports/bibfiles/All}}

\newpage
\section{Appendix}

We show in the following that the iteration complexity
of Algorithm \ref{alg:ALM-nonsmooth} is $O(1/\epsilon)$ to get an $\epsilon$-optimal solution.
First, we need the following definitions and a lemma which is a generalization of Lemma 2.3 in \cite{Beck-Teboulle-2009}.
Let $\Psi: \br^n\rightarrow \br$ and $\Phi:\br^n\rightarrow \br$ be convex functions and define
\beaa Q_\psi(u,v):= \phi(u) + \psi(v) + \langle \gamma_\psi(v), u-v \rangle + \frac{1}{2\mu}\|u-v\|_2^2, \eeaa
where $\gamma_\psi(v)$ is any subgradient in the subdifferential $\partial\psi(v)$ of $\psi(v)$ at the point $v$, and \bea\label{def:p_g(v)} p_\psi(v):=\arg\min_u Q_\psi(u,v).\eea

\begin{lemma}\label{lem:BT2.3}
Let $\Phi(\cdot) = \phi(\cdot)+\psi(\cdot)$. For any  $v$, if
\bea\label{lem:BT2.3-assump-g} \Phi(p_\psi(v)) \leq Q_\psi(p_\psi(v),v), \eea
then for any $u$,
\bea\label{lem:BT2.3-conclude-g} 2\mu(\Phi(u) - \Phi(p_\psi(v)) ) \geq \|p_\psi(v)-u\|^2 - \|v-u\|^2. \eea
\end{lemma}

\begin{proof}
From \eqref{lem:BT2.3-assump-g}, we have \bea\label{lem:BT2.3-proof-eq-1}\ba{lll}\Phi(u) - \Phi(p_\psi(v)) & \geq & \Phi(u) - Q_\psi(p_\psi(v),v) \\
      & = & \Phi(u) - \left(\phi(p_\psi(v)) + \psi(v) + \langle \gamma_\psi(v),p_\psi(v)-v\rangle + \frac{1}{2\mu}\|p_\psi(v)-v\|_2^2\right).\ea\eea
Now since $\phi$ and $\psi$ are convex we have
\bea\label{lem:BT2.3-proof-eq-a} \phi(u) \geq \phi(p_\psi(v)) + \langle u-p_\psi(v), \gamma_\phi(p_\psi(v)) \rangle, \eea
\bea\label{lem:BT2.3-proof-eq-b} \psi(u) \geq \psi(v) + \langle u- v, \gamma_\psi(v)\rangle, \eea where $\gamma_\phi(\cdot)$ is a subgradient of $\phi(\cdot)$ and  $\gamma_\phi(p_\psi(v))$ satisfies the first-order optimality conditions for \eqref{def:p_g(v)}, i.e., \bea\label{def:p_g(v)-1st-opt-cond}\gamma_\phi(p_\psi(v)) + \gamma_\psi(v) + \frac{1}{\mu}(p_\psi(v)-v)=0.\eea Summing \eqref{lem:BT2.3-proof-eq-a} and \eqref{lem:BT2.3-proof-eq-b} yields
\bea\label{lem:BT2.3-proof-eq-2}\Phi(u) \geq \phi(p_\psi(v)) + \langle u-p_\psi(v), \gamma_\phi(p_\psi(v)) \rangle + \psi(v) + \langle u- v, \gamma_\psi(v)\rangle. \eea
Therefore, from \eqref{lem:BT2.3-proof-eq-1}, \eqref{def:p_g(v)-1st-opt-cond} and \eqref{lem:BT2.3-proof-eq-2} it follows that
\begin{align*} \Phi(u) - \Phi(p_\psi(v)) & \geq \langle \gamma_\psi(v)+\gamma_\phi(p_\psi(v)), u-p_\psi(v) \rangle - \frac{1}{2\mu}\|p_\psi(v)-v\|_2^2 \\
                                & = \langle -\frac{1}{\mu}(p_\psi(v)-v), u-p_\psi(v)\rangle - \frac{1}{2\mu}\|p_\psi(v)-v\|_2^2 \\
                                & = \frac{1}{2\mu}\left(\|p_\psi(v)-u\|^2-\|u-v\|^2\right).\end{align*}
\end{proof}

\begin{proof}[Proof of Theorem \ref{the:ALM-nonsmooth}]
Let $I$ be the set of all iteration indices until $k-1$-st for which no skipping
occurs and let $I_c$ be its complement. Let $I=\{n_i\}, \ i=0, \ldots, k_n-1$.
It follows that for all $n\in I_c$ $x^{n+1}=y^n$.

For $n\in I$ we can apply Lemma \ref{lem:BT2.3} to obtain
the following inequalities.
In \eqref{lem:BT2.3-conclude-g}, by letting $\psi=f$, $\phi=g$, $u=x^*$ and $u=x^{n+1}$, we
get $p_\psi(v)=y^{n+1}$, $\Phi=F$ and \bea\label{proof-ALM-1}2\mu(F(x^*)-F(y^{n+1}))\geq
\|y^{n+1}-x^*\|^2-\|x^{n+1}-x^*\|^2.\eea Similarly, by letting $\psi=g$, $\phi=f$,
$u=x^*$ and $v=y^{n}$ in \eqref{lem:BT2.3-conclude-g} we get $p_g(v)=x^{n+1}$, $\Phi=F$ and
\bea\label{proof-ALM-2}2\mu(F(x^*)-F(x^{n+1}))\geq
\|x^{n+1}-x^*\|^2-\|y^{n}-x^*\|^2.\eea Taking the summation of
\eqref{proof-ALM-1} and \eqref{proof-ALM-2} we get
\bea\label{proof-ALM-3}2\mu(2F(x^*)-F(x^{n+1})-F(y^{n+1}))\geq
\|y^{n+1}-x^*\|^2-\|y^n-x^*\|^2.\eea

For $n\in I_c$, \eqref{proof-ALM-1} holds, and we get
\bea\label{proof-ALM-3.5}2\mu(F(x^*)-F(y^{n+1}))\geq
\|y^{n+1}-x^*\|^2-\|y^n-x^*\|^2,\eea
due to the fact that $x^{n+1}=y^n$ in this case.

Summing \eqref{proof-ALM-3} and  \eqref{proof-ALM-3.5}
over $n=0,1,\ldots,k-1$ we get \begin{align}\label{proof-ALM-4}
& 2\mu((2|I|+|I_c|)F(x^*)-\sum_{n\in I}F(x^{n+1}) - \sum_{n=0}^{k-1}F(y^{n+1})) \\
\nonumber\geq & \sum_{n=0}^{k-1}\left(\|y^{n+1}-x^*\|^2-\|y^{n}-x^*\|^2\right) \\
\nonumber = & \|y^k-x^*\|^2-\|y^0-x^*\|^2 \\
\nonumber \geq & - \|x^0-x^*\|^2.\end{align}

For any $n$, since Lemma \ref{lem:BT2.3} holds for any $u$, letting $u=x^{n+1}$ instead of $x^*$ we get from \eqref{proof-ALM-1} that
\bea\label{proof-ALM-5}2\mu(F(x^{n+1})-F(y^{n+1}))\geq
\|y^{n+1}-x^{n+1}\|^2 \geq 0,\eea
 or, equivalently,
\bea\label{proof-ALM-6}2\mu(F(x^{n})-F(y^{n}))\geq
\|y^{n}-x^{n}\|^2 \geq 0.\eea

Similarly, for $n\in I$ by letting
$u=y^{n}$ instead of $x^*$ we get from \eqref{proof-ALM-2} that
\bea\label{proof-ALM-7}2\mu(F(y^{n})-F(x^{n+1}))\geq\|x^{n+1}-y^{n}\|^2 \geq 0.\eea
On the other hand, for $n\in I_c$,  \eqref{proof-ALM-7} also holds because
$x^{n+1}=y^n$, and hence holds for all $n$.

Adding \eqref{proof-ALM-5} and \eqref{proof-ALM-7} we obtain
\bea\label{proof-ALM-9}2\mu(F(y^n)-F(y^{n+1})) \geq 0.\eea 
and adding \eqref{proof-ALM-6} and \eqref{proof-ALM-7} we obtain
\bea\label{proof-ALM-11}2\mu(F(x^{n})-F(x^{n+1}))\geq 0. \eea
\eqref{proof-ALM-9} and \eqref{proof-ALM-11} show that the sequences $F(y^n)$ and $F(x^n)$ are non-increasing. Thus we have,
\bea\label{proof-ALM-12} \sum_{n=0}^{k-1}F(y^{n+1})\geq kF(y^k)\quad \mbox{ and } \quad \sum_{n\in I}F(x^{n+1}) \geq k_n F(x^k).\eea

Combining \eqref{proof-ALM-4} and \eqref{proof-ALM-12} yields
\bea\label{proof-ALM-13}2\mu\left((k+k_n)F(x^*)-k_n F(x^k)-k F(y^k)\right)\geq -\|x^0-x^*\|^2.\eea From \eqref{proof-ALM-6} we know that $F(x^k)\geq F(y^k)$. Thus
\eqref{proof-ALM-13} implies that \beaa 2\mu (k+k_n) \left(F(y^k)-F(x^*)\right) \leq \|x^0-x^*\|^2, \eeaa which gives us the desired result \eqref{the:ALM-nonsmooth-conclude}.

Also, for any given $\epsilon>0$, as long as $k\geq \frac{L(f)\|x^0-x^*\|^2}{2\beta\epsilon}$, we have from \eqref{the:ALM-nonsmooth-conclude} that $F(y^k)-F(x^*)\leq\epsilon$; i.e., the number of iterations needed is $O(1/\epsilon)$ for an $\epsilon$-optimal solution.
\end{proof}

\end{document}